\pdfoutput=1
\documentclass[11pt]{article}
\usepackage{acl}
\usepackage{times}
\usepackage{latexsym}
\usepackage[T1]{fontenc}
\usepackage[utf8]{inputenc}
\usepackage{microtype}
\usepackage{inconsolata}
\usepackage{array,tabularx}
\usepackage[table]{xcolor}
\usepackage{pifont}

\newcommand{\rwcheck}{\textcolor{green!55!black}{\ding{51}}}
\newcommand{\rwxmark}{\textcolor{red!70!black}{\ding{55}}}
\usepackage{graphicx}
\usepackage{adjustbox}
\usepackage{float}
\usepackage{amsmath,amssymb,amsthm,mathtools}
\usepackage{booktabs,multirow}
\usepackage{xcolor,enumitem,pifont}
\usepackage{xurl}
\usepackage{url}
\graphicspath{{figures/}}
\newtheorem{definition}{Definition}

\title{World-Model Collapse as a Phase Transition}

\author{\textbf{Xinyuan Song}$^{1}$ \quad
    \textbf{Zekun Cai}$^{2}$ \\
    $^{1}$Emory University, Atlanta, GA, USA \quad
    $^{2}$The University of Tokyo, Tokyo, Japan \\
    \texttt{xinyuan.song@emory.edu, caizekun@csis.u-tokyo.ac.jp} \\
}

\begin{document}
\maketitle

\begin{abstract}
Water looks unchanged as it warms, then at a critical point it boils. We ask whether long-horizon language agents show an analogous transition in their ability to maintain an accurate representation of the evolving world state. In some parameter settings, a small change in a task parameter, or adding a single step to the task, leaves behavior nearly unchanged; near a critical boundary, the same small change causes a sudden collapse. We study this effect in a deterministic task family where the true world state is known at every step. A large grid search over state size, state dependency, horizon, number of available actions, observation mode, and frequency of state changes reveals a phase diagram: a low-stress region where the agent consistently succeeds, a narrow region where success drops sharply, and a high-stress region where the agent consistently fails. Step-by-step traces show the mechanism: the agent loses track of the true world before its actions become invalid, so the agent is not merely choosing a bad action; it is acting from a corrupted world. Stronger models reach the critical boundary later but do not eliminate the sharp transition. These results make world-model collapse a measurable bottleneck for long-horizon agents. Code will be made publicly available soon.
\end{abstract}

\vspace{-1em}
\section{Introduction}
\vspace{-0.5em}

\label{sec:intro}

Many systems remain stable until a key parameter crosses a critical value. Water can be heated from $10^\circ$C to $99^\circ$C without changing phase, but near $100^\circ$C a small increase causes boiling. We ask whether long-horizon language agents can fail in a similar way. An agent may track a task, update its memory, and choose valid actions for many steps. However, a small increase in state load or state dependency may push its implicit world model beyond a critical point. The resulting failure is not necessarily gradual: the agent may continue to produce plausible reasoning even though its representation of the current world has become incorrect.

This analogy motivates us to study not only whether an agent fails, but also how its success rate changes as task complexity increases. Statistical physics describes phase transitions using control parameters, order parameters, and critical regions~\citep{stanley1973introduction,goldenfeld2018lectures,sethna2021statistical}. We adopt this terminology only as an operational framework and do not claim a thermodynamic-limit phase transition. In our setting, state size and state dependency are the main control parameters, task success is the order parameter, and world-state fidelity provides an early signal of collapse. Our central question is whether performance decreases gradually as these parameters changes or instead decreases sharply near a critical boundary.

This perspective differs from a gradual error-accumulation account. ReAct-style and Reflexion-style agents expose how local reasoning and action errors can propagate across steps~\citep{yao2023react,shinn2023reflexion}, while search methods such as Tree-of-Thoughts and Graph-of-Thoughts can accumulate errors across larger reasoning structures~\citep{yao2023tree,besta2024graph}. Similar long-horizon failures are observed across interactive agent benchmarks~\citep{shridhar2021alfworld,wang2022scienceworld,yao2022webshop,zhou2023webarena,liu2023agentbench} and more complex real-world task suites~\citep{mialon2023gaia,jimenez2023swebench,xie2024travelplanner,yao2024taubench,mazaheri2026agentatlas}. If error accumulation fully explains long-horizon failure, increasing task difficulty should produce a relatively gradual decline in performance. By contrast, if the maintained world state has a capacity boundary, a small increase in task difficulty near that boundary can produce a large drop in performance. Final task success alone cannot distinguish these explanations because the agent may lose an accurate representation of the world before it takes its first invalid action.

We therefore stress the agent's world model directly using two control variables motivated by physical phase transitions, where a small set of control parameters determines the system regime. State size, \textsc{ss}, measures how many state variables the agent must track, while state dependency, \textsc{sd}, measures how strongly these variables depend on one another through action preconditions. We systematically vary these two primary axes and separately test task length, number of available actions, observation mode, and frequency of state changes as alternative drivers. Gradual error accumulation predicts a smooth decrease in success as structural load increases, whereas a phase-transition account predicts a high-success region, a narrow transition region, and a collapse region. It also predicts that world-state fidelity fails before action validity. Figure~\ref{fig:intuition} illustrates this view, while Section~\ref{sec:method} describes the instrumented agent architecture used to measure this ordering.

\begin{figure*}[!ht]
\vspace{-1em}

\centering
\includegraphics[width=.95\textwidth]{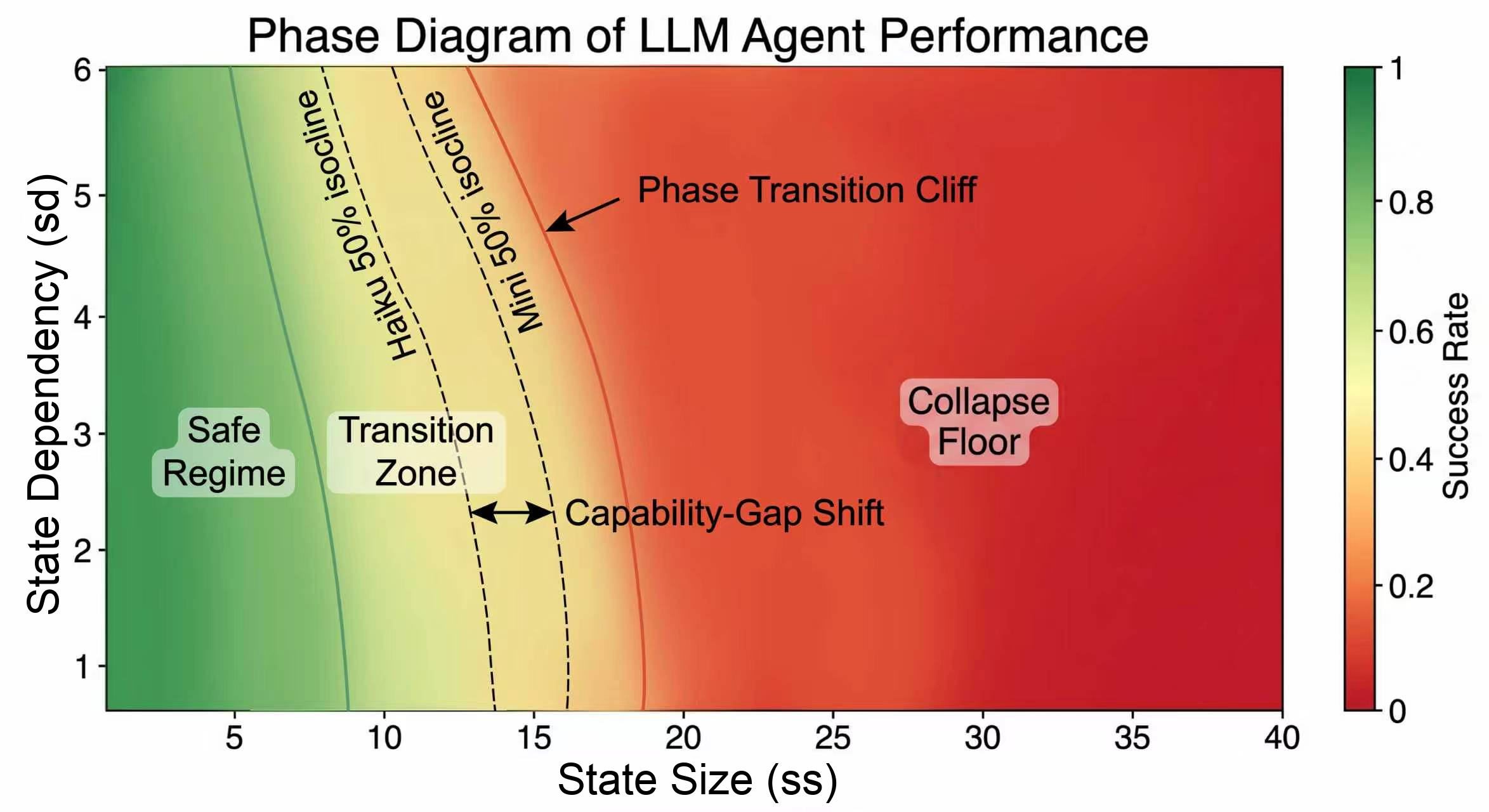}
\vspace{-1em}
\caption{Conceptual phase diagram of world-model collapse. The $(\textsc{ss},\textsc{sd})$ plane illustrates an abrupt transition from reliable to failed behavior as task parameter changes, with stronger models reaching this transition at higher difficulty levels.}
\label{fig:intuition}
\vspace{-1em}
\end{figure*}

We first identify a finite-grid phase-transition pattern in long-horizon language agents across three controlled environments: StatefulPuzzle~\citep{toshniwal2022chess}, GraphNav~\citep{chen2019graphnav}, and ToolDAG~\citep{deelman2004pegasus}. As state size and state dependency increase, success changes sharply from a high-success region to a collapse region, with the transition boundary taking different shapes across environments. We then investigate the mechanism behind this transition by comparing the agent's estimated world state with the true state at each step. The onset analysis shows that world-state fidelity typically degrades before the first invalid action, indicating that collapse begins in the maintained world representation rather than only at action selection. Fine-grained scans localize the model-dependent critical boundary, further characterizing world-model collapse as a measurable phase-transition phenomenon.

Our main contributions are:
\begin{itemize}
    \item We identify a finite-grid phase-transition pattern in long-horizon language agents, characterized by a sharp shift from reliable operation to world-model collapse.
    \item We construct a two-dimensional phase diagram using state size (\textsc{ss}) and state dependency (\textsc{sd}) as control variables and observe the transition across three controlled environments.
    \item We show through per-step analysis that world-state failure generally precedes the first invalid action, revealing a state-first failure mechanism near the transition boundary.
    \item We localize model-dependent critical points and show that changing the model shifts the transition boundary while preserving the overall phase-transition structure.
\end{itemize}

\vspace{-0.5em}
\section{Related Work}
\vspace{-0.5em}

\label{sec:related}

\paragraph{Agent planning and state tracking.}
Language agents are evaluated across interactive environments, web tasks, planning, software engineering, and tool use~\citep{shridhar2021alfworld,wang2022scienceworld,zhou2023webarena,liu2023agentbench,jimenez2023swebench,xie2024travelplanner,yao2024taubench}, while controlled benchmarks isolate specific planning capabilities and task properties~\citep{valmeekam2022planbench,zhao2026planningbench}. Prior work has documented substantial degradation as planning complexity increases~\citep{valmeekam2023planning,kambhampati2024canllmsplan,stechly2024selfverification}, motivating approaches that externalize planning or search over candidate reasoning paths~\citep{kambhampati2024modulo,liu2023llmp,silver2024generalized,yao2023tree,besta2024graph}. Complementary work studies whether language models can maintain and update structured representations of their environment~\citep{chen2023llmstate,samiei2025fsm,zhu2026pddlmind,chen2025worldprediction}, connecting agent behavior to the broader idea of world models as internal representations for prediction and decision making~\citep{ha2018worldmodels,hafner2020dreamerv2,hafner2023dreamerv3}. We connect these directions by asking whether planning failure is preceded by degradation of the maintained world state as structural task load increases.

\paragraph{Structural load and long-horizon interaction.}
Realistic agent benchmarks often change state complexity, interaction length, observations, and interface design simultaneously, making the source of failure difficult to isolate; evaluation outcomes can also depend on the surrounding harness~\citep{kapoor2025agentsmatter}. Long-horizon agents face an additional challenge because persistent memory and context must preserve relevant information across many steps~\citep{park2023generative,packer2024memgpt,sumers2024cognitive}, while long contexts can suffer from information loss and ineffective use of relevant evidence~\citep{liu2023lostinmiddle,fang2026agentlongbench}. We therefore separate structural state load from interaction length: state size and state dependency are varied independently under exact per-step verification, while horizon ablations test whether longer interaction alone produces the same failure pattern.

\paragraph{Transition-like behavior in learning systems.}
Sharp behavioral changes under increasing control parameters have been studied across learning systems, including emergent abilities, grokking, induction-head formation, and double descent~\citep{wei2022emergent,schaeffer2023mirage,power2022grokking,olsson2022induction,belkin2019doubledescent}, drawing on the broader language of phase transitions~\citep{stanley1973introduction,goldenfeld2018lectures}. \citet{nakaishi2024criticalphase} examine transition behavior in static LLM outputs as decoding temperature changes. We instead study interactive agent reliability under task-side structural controls, using success as the behavioral quantity of interest and state size and state dependency to localize an operational finite-grid boundary.

\begin{table*}[!ht]
\centering
\footnotesize
\setlength{\tabcolsep}{2.5pt}
\caption{Comparison with closely related work. ``World-state failure'' denotes
explicit evaluation of the agent's maintained internal state against the true
environment state; ``failure ordering'' denotes comparison of representation
failure with subsequent behavioral failure.}
\vspace{-0.7em}
\label{tab:related_work_novelty}

\resizebox{\textwidth}{!}{%
\begin{tabular}{
    @{}
    >{\raggedright\arraybackslash}p{0.45\textwidth}
    >{\centering\arraybackslash}p{0.09\textwidth}
    >{\centering\arraybackslash}p{0.09\textwidth}
    >{\centering\arraybackslash}p{0.09\textwidth}
    >{\centering\arraybackslash}p{0.09\textwidth}
    >{\centering\arraybackslash}p{0.09\textwidth}
    >{\centering\arraybackslash}p{0.09\textwidth}
    @{}
}
\toprule
\textbf{Work}
& \textbf{Collapse / transition}
& \textbf{Structural control}
& \textbf{World-state failure}
& \textbf{Failure ordering}
& \textbf{Phase boundary}
& \textbf{Cross-model shift} \\
\midrule

The Illusion of Thinking \citep{shojaee2025illusion}
& \rwcheck & \rwcheck & \rwxmark & \rwxmark & \rwcheck & \rwcheck \\

seqBench \citep{ramezanali2025seqbench}
& \rwcheck & \rwcheck & \rwcheck & \rwxmark & \rwxmark & \rwcheck \\

FSM execution \citep{samiei2025fsm}
& \rwxmark & \rwcheck & \rwcheck & \rwxmark & \rwxmark & \rwcheck \\

WMF-AM \citep{hou2026wmfam}
& \rwcheck & \rwcheck & \rwcheck & \rwxmark & \rwcheck & \rwcheck \\

PDDL-Mind \citep{zhu2026pddlmind}
& \rwxmark & \rwxmark & \rwcheck & \rwxmark & \rwxmark & \rwcheck \\

Canonical-path analysis \citep{lee2026canonical}
& \rwxmark & \rwxmark & \rwxmark & \rwxmark & \rwxmark & \rwcheck \\

SIMMER \citep{lu2026simmer}
& \rwxmark & \rwxmark & \rwcheck & \rwxmark & \rwxmark & \rwcheck \\

Temperature-induced criticality \citep{nakaishi2024criticalphase}
& \rwcheck & \rwxmark & \rwxmark & \rwxmark & \rwcheck & \rwxmark \\

\midrule
\rowcolor{blue!9}
\textbf{This work}
& \rwcheck & \rwcheck & \rwcheck & \rwcheck & \rwcheck & \rwcheck \\
\bottomrule
\end{tabular}%
}

\vspace{-0.8em}
\end{table*}

\paragraph{Relation to Prior Work.}
Table~\ref{tab:related_work_novelty} summarizes the key distinctions from prior work: we jointly localize structural success boundaries across state size and dependency and compare world-state failure $\tau_W$ with the first invalid action $\tau_A$ to characterize where collapse occurs and what fails first.

\vspace{-0.5em}
\section{Formal Framework}
\vspace{-0.5em}

\label{sec:formal}

We consider an LLM agent interacting with a deterministic task family controlled by state size $ss\in\mathbb{N}$ and state dependency $sd\in\mathbb{N}$. At each step, the agent receives an observation and selects an action based on its understanding of the current world state. We refer to this maintained understanding as the agent's world model, and denote the model by $\theta$. The vector $z$ contains the remaining experimental factors held fixed, including task length, number of available actions, observation mode, and frequency of state changes. For a fixed model $\theta$ and setting $(s,d,z)$, each seed $\xi_i$ generates a task instance $x_i=x(s,d,z,\xi_i)$. A rollout is one complete agent--environment interaction on $x_i$, ending when the task succeeds, fails, or reaches the horizon limit.

We define a deterministic simulator that executes agent actions, applies the task rules, and records the resulting world state. At step $t$ of rollout $i$, $W_{i,t}^{\ast}$ denotes the true environment state recorded by the deterministic simulator, while $\widehat{W}_{i,t}$ denotes the agent's estimate of that state based on its interaction history. Comparing these two states reveals whether the agent is correctly tracking the true world state. Let $G_i$ denote the goal condition of task instance $x_i$, and let $T_i$ be the final step of rollout $i$. Since $W_{i,T_i}^{\ast}$ is the gold environment state, rollout $i$ is successful if $G_i(W_{i,T_i}^{\ast})=1$. For $n$ rollouts at the same $(s,d,z)$ setting, we define the empirical success rate as
\begin{equation}
\widehat p(s,d;z)=\frac{1}{n}\sum_{i=1}^{n}\mathbf{1}\left\{G_i(W_{i,T_i}^{\ast})=1\right\},
\end{equation}
which serves as the order parameter of the two-dimensional phase diagram.

\begin{definition}[Finite-grid phase transition]
Let $\theta$ denote the fixed world-model and agent configuration, let $z$ contain the fixed task factors including the horizon $H$, and let $\mathcal{G}_s$ and $\mathcal{G}_d$ be the tested grids of state size and state dependency. We write $\widehat p_{\theta}(s,d;z)$ for the empirical success rate of this fixed model--agent system. For a fixed $d\in\mathcal{G}_d$, define the transition critical point
\begin{equation}
\begin{aligned}
&s_{\theta}^{\star}(d;z)=\inf\Bigl\{s\in\mathcal{G}_s:\widehat p_{\theta}(s,d;z)\le\alpha,\\ &\widehat p_{\theta}(s^{-},d;z)\ge\beta, \widehat p_{\theta}(s^{-},d;z)-\widehat p_{\theta}(s,d;z)\ge\delta\Bigr\},
\end{aligned}
\end{equation}
where $s^{-}$ is the immediately preceding tested value of $s$ in $\mathcal{G}_s$. Similarly, for a fixed $s\in\mathcal{G}_s$ and $d^{-}$ is the immediately preceding tested value of $d$ in $\mathcal{G}_d$, define
\begin{equation}
\begin{aligned}
&d_{\theta}^{\star}(s;z)=\inf\Bigl\{d\in\mathcal{G}_d:\widehat p_{\theta}(s,d;z)\le\alpha,\\ &\widehat p_{\theta}(s,d^{-};z)\ge\beta,\widehat p_{\theta}(s,d^{-};z)-\widehat p_{\theta}(s,d;z)\ge\delta\Bigr\},
\end{aligned}
\end{equation}
\end{definition}
To identify what fails near the transition point, we compare the agent's maintained world state with the gold state and track the validity of its actions at every step. With a window size $h$, we define the world-state failure time and action failure time as
\begin{equation}
\begin{aligned}
&\tau_{W,i}=\inf\Bigl\{t\ge h-1:\\
&\qquad\frac{1}{h}\sum_{j=0}^{h-1}\rho(\widehat{W}_{i,t-j},W_{i,t-j}^{\ast})<\frac{1}{2}\Bigr\},
\end{aligned}
\label{eq:tau_w}
\end{equation}
\begin{equation}
\tau_{A,i}=\inf\left\{t\ge1:a_{i,t}\text{ is invalid in }W_{i,t}^{\ast}\right\}.
\label{eq:tau_a}
\end{equation}
where $\rho(\widehat{W}_{i,t},W_{i,t}^{\ast})$ measures agreement between the agent-maintained and gold states. Here, $\tau_{W,i}$ is the first sustained loss of world-state fidelity, while $\tau_{A,i}$ is the first invalid action. Comparing $\tau_{W,i}$ and $\tau_{A,i}$ therefore determines whether world-state failure precedes behavioral failure, which we use to diagnose the mechanism underlying the observed phase transition.

\vspace{-1em}
\section{Experimental Setup}
\vspace{-0.5em}

\label{sec:method}

We use controlled deterministic environments with exact per-step gold states to isolate structural drivers of world-model collapse from factors such as horizon, observation noise, and model size. This design follows controlled planning and agent benchmarks~\citep{valmeekam2022planbench,zhao2026planningbench,fang2026agentlongbench} and recent work on reproducible agent evaluation~\citep{kapoor2025agentsmatter}.

\paragraph{Environments.}
\label{sec:method:envs}
We build three deterministic simulators with randomness. \textbf{GraphNav}~\citep{chen2019graphnav} tests planning and state tracking in a room-and-door graph. The agent must navigate between rooms while tracking its current location, available keys, switch states, and door conditions, including irrelevant decoy paths. Increasing \textsc{ss} enlarges the graph and the amount of state that must be maintained, while increasing \textsc{sd} adds more conditions that must be satisfied before a navigation action is valid.\textbf{ToolDAG}~\citep{deelman2004pegasus} tests planning over a directed acyclic graph of typed tool calls, where each tool consumes variables produced by earlier calls and may generate new variables for later steps. Increasing \textsc{ss} enlarges the active variable set that must be tracked, while increasing \textsc{sd} adds more input dependencies that must be satisfied before a tool can be executed.\textbf{StatefulPuzzle}~\citep{toshniwal2022chess} asks the agent to move objects across rooms, containers, and slots through ordered subgoals. In StatefulPuzzle, \textsc{ss} counts the rooms, containers, and items that must be jointly tracked, while \textsc{sd} counts the preconditions connecting each subgoal to the current world state. Formal environment definitions and reproducibility details are provided in Appendix~\ref{app:envs}.

\paragraph{Confirmatory-environment selection.}
\label{sec:method:trigger}
For each candidate environment, we run a one-dimensional \textsc{sd} pilot sweep with $n{=}10$ rollouts per level for two model tiers. A candidate is selected only if all three pre-specified criteria are satisfied: (a) success changes monotonically with \textsc{sd} for each model, (b) both models show the same qualitative trend, and (c) the smaller of their maximum adjacent success-rate drops satisfies $\min(\max\Delta\hat{p})\ge20\%$. StatefulPuzzle is the only candidate satisfying all three criteria and is therefore used for the confirmatory experiments.

\paragraph{Stress grid.}
\label{sec:method:grid}
The main grid covers $\textsc{ss}\in\{5,10,20,40\}$ and $\textsc{sd}\in\{1,2,4,6\}$, with horizon $T{=}40$, branching factor $4$, clean observations, and static mutation fixed at Regime~III. Each of the 16 cells contains $n{=}100$ rollouts from 10 archetypes $\times$ 10 instance variants, using SHA-256-derived seeds (1{,}600 unique seeds with zero collisions).

\paragraph{Agent architecture.}
\label{sec:method:agent}
We use a language-agent architecture combining interleaved reasoning and action, self-reflection, and structured memory~\citep{yao2023react,shinn2023reflexion,chen2023llmstate,packer2024memgpt}. Each interaction step consists of three model calls: a \textbf{Planner}, an \textbf{Updater}, and \textbf{Self-Diag}. The Planner reads the maintained structured world state and selects an action together with its expected preconditions and effects. After execution, the Updater incorporates the new observation, adds newly observed facts, removes stale facts, and produces the world state for the next step.

Self-Diag independently checks the selected action for validity and precondition errors but is non-blocking: it does not modify the action, memory, or environment state. The structured representation can therefore be compared directly with the simulator's exact gold state at every step, enabling measurement of world-state fidelity and its temporal relation to action failure. The architecture, memory format, and component interfaces are fixed across environments and experimental conditions; malformed outputs are retried a limited number of times and otherwise handled by fixed rules. The main grid uses claude-haiku-4-5 at temperature~$0$~\citep{anthropic2025haiku45}.

\paragraph{Acceptance criteria.}
\label{sec:method:g1protocol}
We use two primary criteria to evaluate the collapse transition. The \textbf{cliff criterion} tests whether adjacent grid settings exhibit an abrupt drop in success rate. Specifically, we use the one-sided Miettinen--Nurminen score test~\citep{miettinen1985comparative} for $H_0{:}\Delta p\leq0.30$ versus $H_{\mathrm{alt}}{:}\Delta p>0.30$ at $\alpha=0.01$, where $\Delta p=\hat p_1-\hat p_2$ is the success-rate drop between adjacent settings. The test statistic is
\begin{equation}
\small
T =
\frac{(\hat p_1-\hat p_2)-\delta_0}
{\sqrt{\tilde p_1(1-\tilde p_1)/n_1+\tilde p_2(1-\tilde p_2)/n_2}}, \delta_0=0.30,
\end{equation}
where $(\tilde p_1,\tilde p_2)$ are the constrained maximum-likelihood estimates under $H_0{:}p_1-p_2=\delta_0$. Thus, the criterion requires evidence for an adjacent drop exceeding 30 percentage points rather than merely a gradual decline. The main cliff conclusions remain significant under both Bonferroni correction at $\alpha=0.01$ and Benjamini--Hochberg correction at $q=0.05$.

The \textbf{failure-order criterion} compares the onset times $\tau_W$ and $\tau_A$ defined in Section~\ref{sec:formal}. For each collapsed rollout in which both are observed, we classify the outcome as world-state first ($\tau_W<\tau_A$), simultaneous ($\tau_W=\tau_A$), or action first ($\tau_W>\tau_A$), and report the median lead $\tau_A-\tau_W$.

We additionally use secondary analyses to characterize the transition. SS-Fine localizes the state-size boundary, T-Fine examines horizon effects, and cross-model experiments assess boundary shifts. After correction, horizon and observation effects remain significant, branching shows no clear effect, and mutation is analyzed descriptively.

\section{Results}
\label{sec:results}

\subsection{Main Results Across Three Environments}
\vspace{-0.5em}
\label{sec:results:main}

\paragraph{StatefulPuzzle.}
Figure~\ref{fig:stateful_crossmodel} shows the clearest state-size transition. For claude-haiku-4-5~\citep{anthropic2025haiku45}, success is perfect at $\textsc{ss}{=}5$, decreases across \textsc{sd} at $\textsc{ss}{=}10$, and reaches zero at $\textsc{ss}\ge20$. For gpt-4o-mini~\citep{openai2024gpt4omini}, high success extends further into the grid, with the main transition moving toward $\textsc{ss}{=}20$.

\begin{figure*}[!ht]
\centering
\includegraphics[width=.9\textwidth]{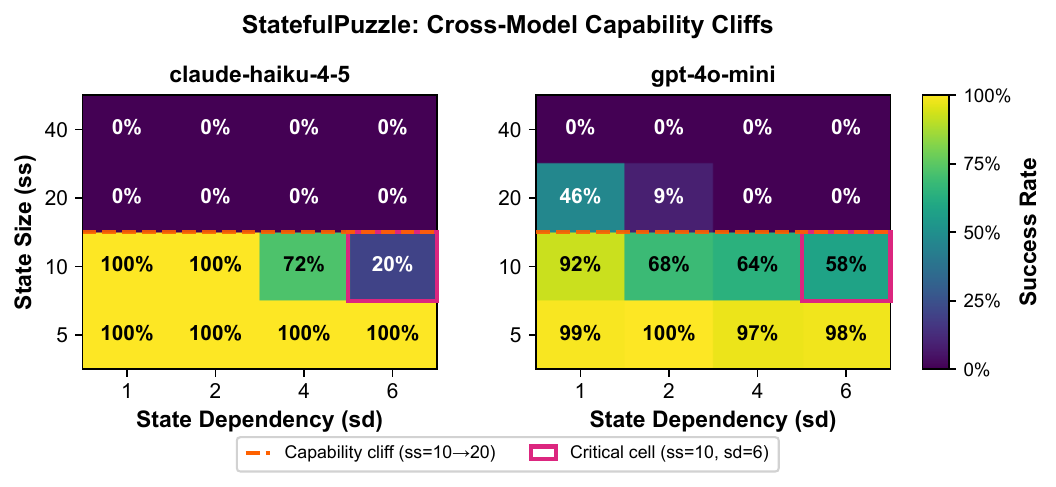}
\vspace{-1em}
\caption{StatefulPuzzle phase diagrams for claude-haiku-4-5 and gpt-4o-mini. The transition is primarily associated with increasing state size, while the boundary differs across models.}
\label{fig:stateful_crossmodel}
\vspace{-1em}
\end{figure*}

\begin{table*}[!ht]
\centering
\small
\setlength{\tabcolsep}{3pt}
\caption{Temporal ordering of world-state and action failure near the observed cliff boundary. Statistics are computed over collapsed rollouts near the transition region in which both onset times are observed. ``World first'' denotes $\tau_W<\tau_A$, ``Same step'' denotes $\tau_W=\tau_A$, and ``Action first'' denotes $\tau_W>\tau_A$. Median lead is $\tau_A-\tau_W$, with positive values indicating earlier world-state failure.}
\vspace{-1em}
\label{tab:failure_order}
\begin{adjustbox}{max width=\textwidth}
\begin{tabular}{l l r r r r r}
\toprule
\textbf{Environment} & \textbf{Model} & \textbf{$n$} & \textbf{World first} & \textbf{Same step} & \textbf{Action first} & \textbf{Median lead} \\
\midrule
StatefulPuzzle & claude-haiku-4-5 & 82 & 78\% & 22\% & 0\% & 2 steps \\
StatefulPuzzle & gpt-4o-mini      & 67 & 75\% & 24\% & 1\% & 2 steps \\
GraphNav       & claude-haiku-4-5 & 76 & 72\% & 25\% & 3\% & 2 steps \\
GraphNav       & gpt-4o-mini      & 61 & 71\% & 26\% & 3\% & 2 steps \\
ToolDAG        & claude-haiku-4-5 & 88 & 76\% & 22\% & 2\% & 2 steps \\
ToolDAG        & gpt-4o-mini      & 73 & 73\% & 25\% & 2\% & 2 steps \\
\bottomrule
\end{tabular}
\end{adjustbox}
\vspace{-1em}
\end{table*}

\begin{figure*}[!ht]
\centering
\includegraphics[width=.9\textwidth]{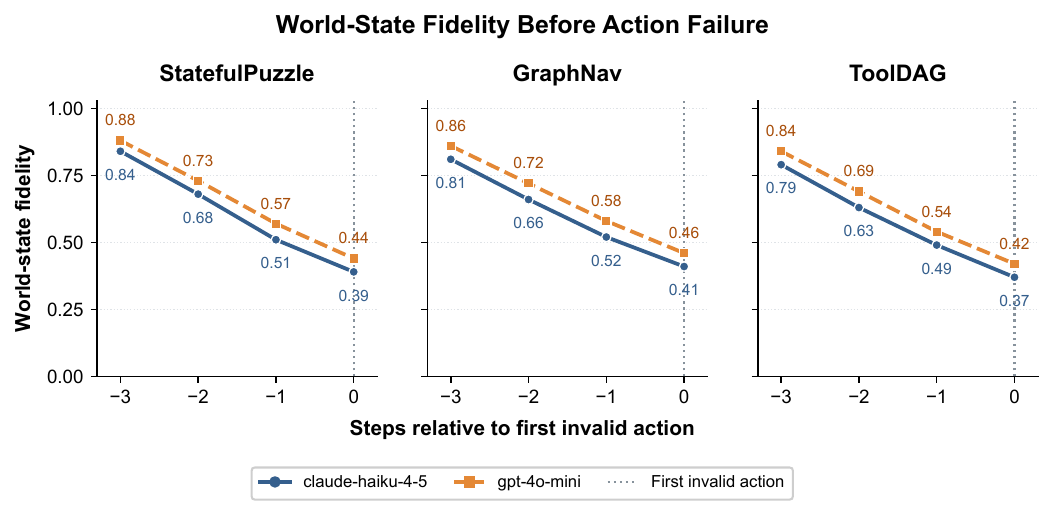}
\vspace{-1em}
\caption{World-state fidelity in the three steps preceding the first invalid action. Collapsed rollouts near the observed cliff boundary are aligned at $\tau_A=0$, with negative values denoting steps before the first invalid action. Across all three environments and both models, fidelity decreases as the rollout approaches action failure.}
\label{fig:fidelity_before_action}
\vspace{-1em}
\end{figure*}

\paragraph{GraphNav.}
Figure~\ref{fig:graphnav_crossmodel} shows a stronger interaction between the two structural axes. Success decreases with both increasing \textsc{ss} and increasing \textsc{sd}, producing a more diagonal transition region. gpt-4o-mini maintains higher success across several intermediate cells, shifting the observed boundary toward larger structural loads.  

\begin{figure*}[!ht]
\centering
\includegraphics[width=.9\textwidth]{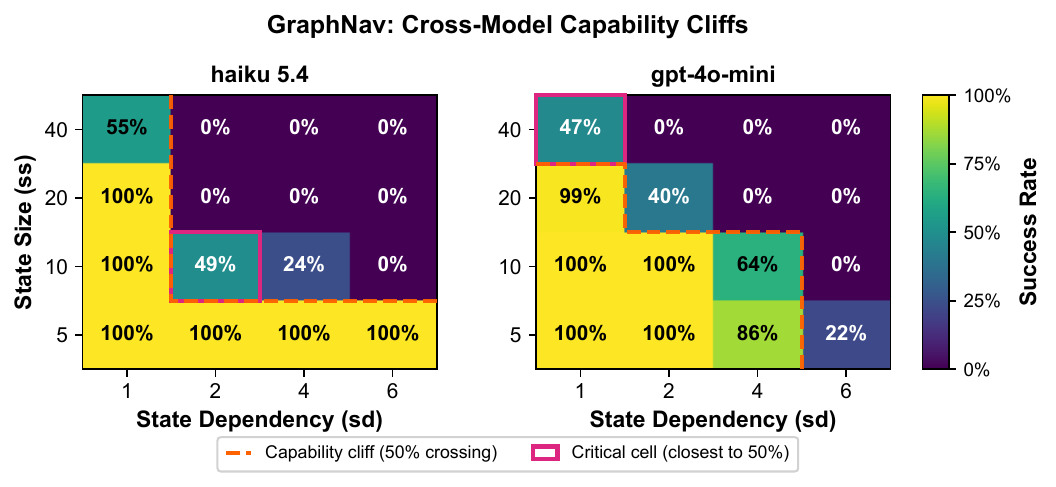}

\vspace{-1em}
\caption{GraphNav phase diagrams for claude-haiku-4-5 and gpt-4o-mini. Both state size and state dependency contribute to the transition.}
\label{fig:graphnav_crossmodel}

\vspace{-1em}
\end{figure*}

\paragraph{ToolDAG.}
Figure~\ref{fig:tooldag_crossmodel} shows an earlier transition. state dependency already reduces success at $\textsc{ss}{=}5$, while increasing \textsc{ss} rapidly moves both models into the failure region. gpt-4o-mini expands the successful region, but both models reach zero success throughout the tested $\textsc{ss}\ge20$ region.  

\begin{figure*}[!ht]
\centering
\includegraphics[width=.9\textwidth]{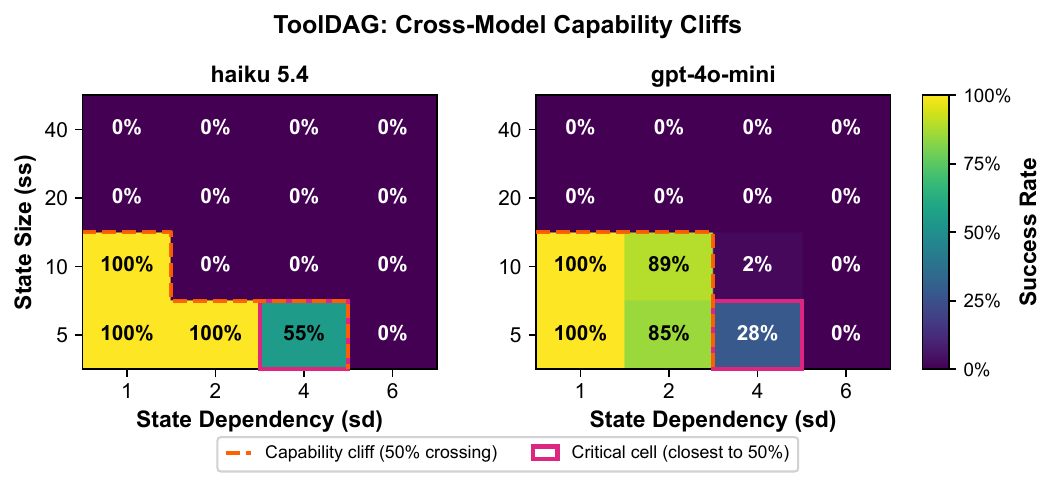}

\vspace{-1em}
\caption{ToolDAG phase diagrams for claude-haiku-4-5 and gpt-4o-mini. The transition occurs at lower structural load than in StatefulPuzzle and GraphNav.}
\label{fig:tooldag_crossmodel}

\vspace{-1em}
\end{figure*}

\paragraph{Cross-environment analysis.}
The three environments show different boundary geometries but the same overall pattern: increasing structural load moves the agent from a high-success region through a transition region toward failure. StatefulPuzzle is mainly separated by \textsc{ss}, GraphNav shows a stronger joint dependence on \textsc{ss} and \textsc{sd}, and ToolDAG reaches the boundary earlier. Across all three environments, changing the model shifts the observed boundary without removing this transition structure.

\vspace{-0.5em}
\subsection{World-State Failure Precedes Action Failure}
\vspace{-0.5em}
\label{sec:results:failure-order}

We next test what fails first when the agent enters the collapse regime. For each collapsed rollout, we record the onset of world-state failure $\tau_W$ (Equation~\ref{eq:tau_w}), defined from per-step world-state fidelity, and the first invalid action $\tau_A$ (Equation~\ref{eq:tau_a}). We evaluate this ordering across StatefulPuzzle, GraphNav, and ToolDAG for both claude-haiku-4-5 and gpt-4o-mini.

Table~\ref{tab:failure_order} shows a consistent temporal ordering across environments and models. World-state failure occurs first in most paired collapsed rollouts, whereas action-first failures are rare. The median lead is two steps in every setting, indicating that degradation of the maintained world state is typically detectable before the first invalid action.

To examine this progression directly, we align collapsed rollouts at $\tau_A$ and measure world-state fidelity during the preceding three steps. Table~\ref{tab:fidelity_before_action} reports the numerical values, and Figure~\ref{fig:fidelity_before_action} visualizes the trajectories.

\begin{table}[!ht]
\centering
\footnotesize
\setlength{\tabcolsep}{2pt}
\caption{World-state fidelity immediately before the first invalid action near the observed cliff boundary. Collapsed rollouts are aligned at $\tau_A$, and each entry reports the mean world-state fidelity $\rho(\widehat{W}_t,W_t^\ast)$ at the corresponding relative step.}
\vspace{-1em}
\label{tab:fidelity_before_action}
\begin{adjustbox}{max width=\columnwidth}
\begin{tabular}{l l c c c c}
\toprule
\textbf{Environment} & \textbf{Model} & $\tau_A-3$ & $\tau_A-2$ & $\tau_A-1$ & $\tau_A$ \\
\midrule
StatefulPuzzle & claude-haiku-4-5 & 84\% & 68\% & 51\% & 39\% \\
StatefulPuzzle & gpt-4o-mini      & 88\% & 73\% & 57\% & 44\% \\
GraphNav       & claude-haiku-4-5 & 81\% & 66\% & 52\% & 41\% \\
GraphNav       & gpt-4o-mini      & 86\% & 72\% & 58\% & 46\% \\
ToolDAG        & claude-haiku-4-5 & 79\% & 63\% & 49\% & 37\% \\
ToolDAG        & gpt-4o-mini      & 84\% & 69\% & 54\% & 42\% \\
\bottomrule
\end{tabular}
\end{adjustbox}
\vspace{-1em}
\end{table}

Across all three environments and both models, world-state fidelity decreases steadily as the rollout approaches $\tau_A$ (Table~\ref{tab:fidelity_before_action}; Figure~\ref{fig:fidelity_before_action}). Together with the onset ordering in Table~\ref{tab:failure_order}, this provides two complementary views of the same failure sequence: degradation of the maintained world state generally emerges before invalid behavior.

\vspace{-0.5em}
\subsection{Critical-Point Localization}
\vspace{-0.5em}
\label{sec:results:localization}

We next localize the transition more precisely along the state-size axis and test whether its location changes across model APIs. The coarse phase diagram identifies the transition region, while a finer \textsc{ss} scan at fixed $\textsc{sd}{=}1$ estimates the corresponding critical point $\textsc{ss}^{\star}$. For claude-haiku-4-5~\citep{anthropic2025haiku45}, the fine scan localizes the transition at $\textsc{ss}^{\star}\approx14$ (Figure~\ref{fig:critical_points}). We then apply the same localization procedure and agent setup to gpt-4o-mini~\citep{openai2024gpt4omini}, GPT-4o~\citep{openai2024gpt4osystem}, and Llama-3 70B~\citep{grattafiori2024llama3}. The resulting $\textsc{ss}^{\star}$ estimates differ across model APIs (Table~\ref{tab:scstar_readout}), showing that the location of the transition is model-dependent while being measured by the same threshold-crossing criterion.

\begin{figure}[!ht]
\centering
\includegraphics[width=0.9\linewidth]{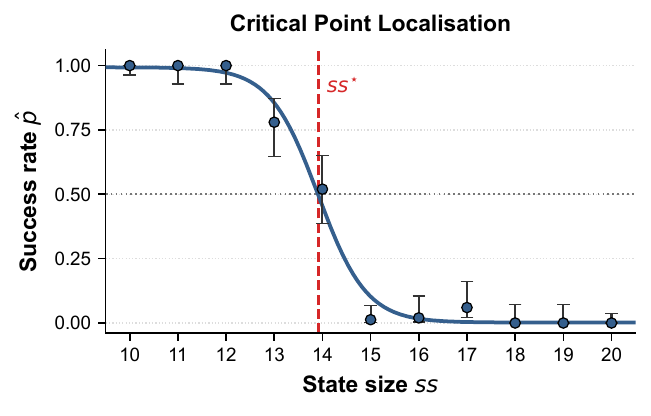}
\vspace{-1.5em}
\caption{Fine scan over state size at fixed $\textsc{sd}{=}1$, localizing the claude-haiku-4-5 transition at $\textsc{ss}^{\star}\approx14$.}
\label{fig:critical_points}
\vspace{-1em}
\end{figure}

\begin{table}[!ht]
\centering
\small
\setlength{\tabcolsep}{3pt}
\caption{Critical-point estimates obtained using the same $\textsc{ss}$ localization procedure at $\textsc{sd}{=}1$.}
\vspace{-1em}
\label{tab:scstar_readout}
\begin{adjustbox}{max width=\columnwidth}
\begin{tabular}{l c}
\toprule
\textbf{Model} & $\textsc{ss}^{\star}$ \\
\midrule
claude-haiku-4-5~\citep{anthropic2025haiku45} & $14.04$ \\
gpt-4o-mini~\citep{openai2024gpt4omini} & $18.94$ \\
GPT-4o~\citep{openai2024gpt4osystem} & $15.56$ \\
Llama-3 70B~\citep{grattafiori2024llama3} & $9.8$ \\
\bottomrule
\end{tabular}
\end{adjustbox}
\vspace{-1em}
\end{table}

\vspace{-0.5em}
\section{Future Work: Multiple Agent Phases}
\label{app:future-phases}
\begin{figure}[!ht]
\centering
\includegraphics[width=0.8\linewidth]{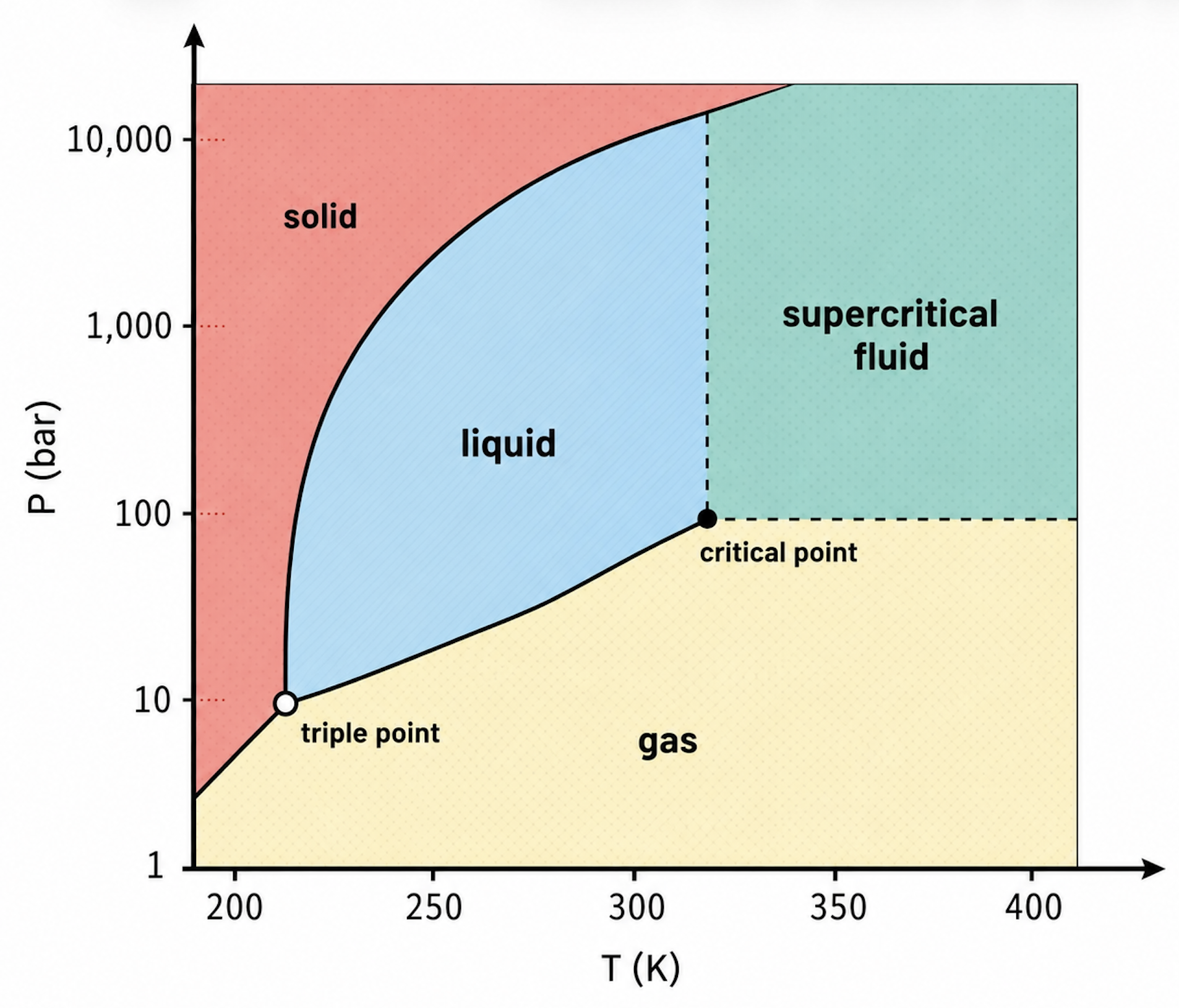}
\vspace{-1em}
\caption{Schematic phase diagram of water. Different control parameters define multiple phases separated by transition boundaries, providing a conceptual analogy for possible multi-phase behavior in language agents.}
\label{fig:water_phase_diagram}
\vspace{-2em}
\end{figure}

Our current analysis uses task success as a binary order parameter, revealing a high-success region, a transition region, and a collapse region. While this formulation makes the collapse boundary measurable, it may hide finer behavioral and representational states within these regions.

Physical phase diagrams suggest a natural extension. As Figure~\ref{fig:water_phase_diagram} illustrates, different control parameters can produce multiple phases separated by distinct boundaries. Language agents may similarly exhibit more than two regimes as model capacity, memory, observation quality, horizon, state size, and state dependency vary.

Future work could therefore define richer order parameters using world-state fidelity, action validity, recovery, uncertainty, and planning consistency. This may reveal distinct regimes such as stable state tracking, recoverable state errors, persistent representation errors, unstable planning, and complete world-model collapse. Mapping these regimes could extend the binary success--collapse transition studied here into a richer phase diagram of long-horizon agent behavior.

\vspace{-0.5em}
\section{Discussion}
\label{sec:discussion}
\vspace{-0.5em}
Our results show that increasing state size and dependency produces a sharp transition from reliable behavior to collapse, with world-state failure typically preceding invalid actions. Cross-model experiments shift this boundary without removing the pattern, suggesting that state maintenance is an important target alongside planning, including structured memory~\citep{packer2024memgpt,park2023generative,sumers2024cognitive} and explicit world-state representations~\citep{zhu2026pddlmind}. We interpret this boundary as a finite-grid phase transition, without claiming thermodynamic-limit behavior.

\vspace{-0.5em}
\section{Conclusion}
\vspace{-0.5em}
\label{sec:conclusion}

We identify a phase-transition pattern in long-horizon LLM agents: as state size and state dependency cross a critical region, performance shifts sharply from reliable success to collapse. Fine scans localize this boundary, and cross-model experiments show that its location changes across models while the transition structure persists. Per-step analysis further shows that world-state failure generally precedes invalid actions, linking the transition to breakdown of the agent's maintained world representation.

\section*{Limitations}
\label{sec:limitations}

\paragraph{Controlled environments.}
Our experiments use controlled simulated environments rather than naturalistic agent benchmarks such as WebArena~\citep{zhou2023webarena}, SWE-Bench~\citep{jimenez2023swebench}, or TravelPlanner~\citep{xie2024travelplanner}. This choice is motivated by the search required to identify a phase-transition boundary: the relevant control variables and transition regions are not known in advance, requiring broad sweeps over state size, state dependency, horizon, branching, observation mode, and mutation rate, followed by dense sampling near candidate boundaries. Repeating such searches across models on benchmarks with long contexts, retrieved documents, or large tool outputs would require substantially more inference tokens and API cost. Lightweight deterministic simulators make this parameter search and per-step state measurement computationally feasible. Extending the analysis to naturalistic agent benchmarks is therefore an important direction for future work.

\paragraph{Scope of the observed transition.}
Our results are obtained with a three-call Planner/Updater/Self-Diag architecture using structured working memory and a limited set of model APIs. The observed boundary locations should therefore be interpreted as properties of the tested model--agent configurations rather than universal capacity limits. Although we observe related transition patterns across StatefulPuzzle, GraphNav, and ToolDAG, broader environments, agent architectures, memory systems, and models may produce different boundary locations or geometries.

\paragraph{Finite-grid phase transition.}
We use ``phase transition'' in an operational finite-grid sense. Across the tested $(\textsc{ss},\textsc{sd})$ grids, success changes from a high-success region through a transition region to collapse, and finer scans localize the state-size boundary at higher resolution. We do not claim a thermodynamic-limit phase transition, estimate critical exponents, or derive scaling laws. The identified critical points are empirical estimates within the tested configurations and parameter ranges.

% \section*{AI Use Statement}

% AI tools assisted with LaTeX template conversion, main-text and appendix organization, formatting checks, and retrieval and discovery of related work.

\bibliography{REFS}
\clearpage
\appendix

\section{Environment Specifications}
\label{app:envs}

All three environments are deterministic, stateful task simulators designed to test whether an agent can maintain an accurate world model while planning over multiple dependent steps. Each rollout is generated from a configuration and seed and proceeds through a sequence of discrete actions. After every action, the simulator records the complete gold world state and whether the selected action satisfies its true preconditions. This common interface allows us to measure both final task success and the step at which the agent's maintained state first diverges from the environment. Across environments, \textsc{ss} controls the amount of state that must be jointly maintained, while \textsc{sd} controls how strongly action validity depends on combinations of that state.

\paragraph{Graph Navigation (GraphNav).}
GraphNav is a controlled navigation environment related to prior graph-based navigation settings~\citep{chen2019graphnav}. The environment consists of rooms connected by doors whose accessibility can depend on keys, switches, and other state-dependent conditions. The agent must reach a target room while tracking its current location, acquired keys, switch states, door conditions, and irrelevant decoy paths. An action becomes invalid when its required navigation conditions are not satisfied in the current gold state. Increasing \textsc{ss} enlarges the environment and the set of state variables that must be maintained simultaneously. Increasing \textsc{sd} increases the number of conditions on which navigation decisions depend. GraphNav therefore tests whether the agent can preserve and use a growing structured state while navigating through dependent decisions.

\paragraph{Tool-DAG (ToolDAG).}
ToolDAG is a controlled tool-use environment following classical DAG-based workflow settings~\citep{deelman2004pegasus}. A task consists of a directed acyclic graph of tool calls, where tools consume typed variables produced by earlier calls and may generate variables required by later calls. The agent must determine which tools are currently executable, select a valid call, and maintain the intermediate variables produced throughout the workflow until the target output is obtained. A tool call is invalid when one or more of its required inputs are unavailable in the current gold state. Increasing \textsc{ss} increases the number of active variables and intermediate results that must be tracked, while increasing \textsc{sd} increases the number of dependencies that must be satisfied before a tool can be executed. ToolDAG therefore tests world-state maintenance under explicit dependency chains.

\paragraph{Stateful Puzzle (StatefulPuzzle).}
StatefulPuzzle is a controlled object-manipulation environment motivated by prior work on language-model state tracking~\citep{toshniwal2022chess}. Tasks contain rooms, containers, slots, and movable objects that must be manipulated through ordered subgoals. The agent must track where each relevant object is located, which containers or locations are available, which earlier subgoals have been completed, and which conditions currently hold before selecting the next action. An action is invalid when its required state-dependent preconditions are not satisfied. Increasing \textsc{ss} increases the number of state variables and state facts that must be represented jointly, while increasing \textsc{sd} increases the number of preconditions linking each subgoal to the current world state. StatefulPuzzle therefore directly stresses the agent's ability to maintain a coherent world representation across a sequence of dependent state changes.

\end{document}